\documentclass{article}

\PassOptionsToPackage{numbers, compress}{natbib}
\usepackage[final]{neurips_2024}

\usepackage[utf8]{inputenc} 
\usepackage[T1]{fontenc}    
\usepackage{hyperref}       
\hypersetup{
    colorlinks=true,
    linkcolor=blue,
    filecolor=magenta,      
    urlcolor=cyan,
    }
\usepackage{url}            
\usepackage{booktabs}       
\usepackage{amsfonts}       
\usepackage{nicefrac}       
\usepackage{microtype}      
\usepackage[html]{xcolor}         

\usepackage{amsthm}
\usepackage{todonotes}
\usepackage{enumitem}

\usepackage{amsmath}
\usepackage{natbib} 
\usepackage{amssymb}
\usepackage{parskip}
\usepackage{subcaption}
\usepackage{titlesec}

\newtheorem{theorem}{Theorem}

\newtheorem{corollary}{Corollary}
\newtheorem{proposition}{Proposition}
\newtheorem{remark}{Remark}

\usepackage{makecell}
\usepackage{multirow}
\usepackage{booktabs}

\usepackage{algorithm}
\usepackage{algpseudocode}

\usepackage{wrapfig}
\usepackage{soul}
\usepackage{mathtools}

\title{Scalable Kernel Inverse Optimization}

\author{%
  Youyuan Long \\
  Delft Center for Systems and Control\\
  Delft University of Technology\\
  The Netherlands \\
  \texttt{longyouyuan432@gmail.com} \\
  \And
  Tolga Ok \\
  Delft Center for Systems and Control\\
  Delft University of Technology\\
  The Netherlands \\
  \texttt{T.Ok@tudelft.nl} \\
  \And
  Pedro Zattoni Scroccaro \\
  Delft Center for Systems and Control\\
  Delft University of Technology\\
  The Netherlands \\
  \texttt{P.ZattoniScroccaro@tudelft.nl} \\
  \And
  Peyman Mohajerin Esfahani \\
  Delft Center for Systems and Control\\
  Delft University of Technology\\
  The Netherlands \\
  \texttt{P.MohajerinEsfahani@tudelft.nl} \\
}

\begin{document}

\maketitle

\begin{abstract}

Inverse Optimization (IO) is a framework for learning the unknown objective function of an expert decision-maker from a past dataset.
In this paper, we extend the hypothesis class of IO objective functions to a reproducing kernel Hilbert space (RKHS), thereby enhancing feature representation to an infinite-dimensional space.
We demonstrate that a variant of the representer theorem holds for a specific training loss, allowing the reformulation of the problem as a finite-dimensional convex optimization program.
To address scalability issues commonly associated with kernel methods, we propose the Sequential Selection Optimization (SSO) algorithm to efficiently train the proposed Kernel Inverse Optimization (KIO) model.
Finally, we validate the generalization capabilities of the proposed KIO model and the effectiveness of the SSO algorithm through learning-from-demonstration tasks on the MuJoCo benchmark.
 
\end{abstract}

\section{Introduction}
Inverse Optimization (IO) is distinct from traditional optimization problems, where we typically seek the optimal decision variables by optimizing an objective function over a set of constraints.
In contrast, inverse optimization works ``in reverse'' by inferring the optimization objective given the optimal solution.
The inherent assumption in IO is that an agent generates its decision by solving an optimization problem.
The assumed optimization problem is called the Forward Optimization Problem (FOP), which is parametric in the exogenous signal $\hat{s}$ with the corresponding optimal solution $\hat{u}$.
Therefore, IO aims to deduce the objective function of the FOP from a dataset of exogenous signal and decision pairs, $\{(\hat{s}_i,\hat{u}_i)\}_{i=1}^N$.
In this work, we assume the constraints are known a priori.
Consequently, we can leverage the FOP derived from the expert’s dataset by solving it to mimic the expert's behavior when encountering new exogenous signals.
IO has garnered widespread attention within several fields, giving rise to numerous studies encompassing both theoretical and applied research.
Application domains include vehicle routing \citep{chen2021inverse, ronnqvist2017calibrated, zattoniscroccaro2024inverse}, transportation system modeling \citep{patriksson2015traffic, bertsimas2015data}, portfolio optimization \citep{mohajerin2018data, yu2023learning, li2021inverse}, power systems \citep{birge2017inverse, fernandez2021forecasting, saez2016data}, electric vehicle charging problems \citep{fernandez2021inverse}, network design \citep{farago2003inverse}, healthcare problems \citep{ayer2015inverse}, as well as controller design \citep{akhtar2021learning}.
For a more detailed discussion on different applications of IO, we refer the readers to the recent survey paper \citep{chan2023inverse}.

IO can be categorized into classic IO and data-driven IO.
In classic IO, only a single signal-decision pair is considered, where the decision is assumed to be the optimal solution of the FOP (i.e., there is no noise), and different classes of FOPs have been studied, such as linear conic problems \cite{ahuja2001inverse, shahmoradi2022quantile, iyengar2005inverse}.
However, in real-world applications, there are usually many observations of signal-decision pairs, and due to the presence of noise, it is usually unreasonable to assume that all observed decisions are optimal w.r.t a single, true data-generating FOP.
Additionally, for complex tasks, the chosen FOP may only approximate the task, not allowing for perfect replication of the observed behavior from the expert.
These cases are referred to as data-driven IO problems.
In such scenarios, a loss function is usually used to compute the discrepancy between observed data and the decision generated by the learned FOP.
Examples of loss functions include the \textit{2-norm distance loss} \citep{aswani2018inverse}, \textit{suboptimality loss} \citep{mohajerin2018data}, \textit{variational inequality loss} \citep{bertsimas2015data}, \textit{KKT loss} \citep{keshavarz2011imputing}, and \textit{augmented suboptimality loss} \citep{zattoniscroccaro2024learning}.

In data-driven IO, the objective function of the FOP is typically non-linear with respect to an exogenous signal $\hat{s}$.
Hence, classical methods that learn an FOP based on linear function classes may oversimplify the problem and lead to suboptimal solutions.
One effective approach for addressing the expressibility issue in data-driven IO problems is the introduction of kernel methods.
These methods have been extensively studied within the context of IO \cite{shafieezadeh2019regularization, bertsimas2015data} and have shown promising results for scaling IO to address practical problems.
The application of kernel methods in IO allows for the exploration of a broader class of optimization problems, thereby enhancing the model’s ability to generalize from observed decisions to unseen situations.
Specifically, using a kernelized approach facilitates the embedding of decision data into a richer feature space, enabling the deduction of an FOP that not only fits the training data but also exhibits strong generalization capabilities.

\textbf{Contributions}.
We list the contributions of this work as follows:

\begin{enumerate}[label={\textbf{(\arabic*)}}]
    \item \textbf{Kernelized IO Formulation}: We propose a novel Kernel Inverse Optimization (KIO) model based on suboptimality loss \cite{mohajerin2018data}.
    The proposed approach leverages kernel methods to enable IO models to operate on infinite-dimensional feature spaces, which allows KIO to outperform existing imitation learning (IL) algorithms on complex continuous control tasks in low-data regimes.
    
    \item \textbf{Sequential Selection Optimization Algorithm}: To address the quadratic computational complexity of the proposed KIO model, we introduce the Sequential Selection Optimization (SSO) algorithm inspired by coordinate descent style updates.
    This algorithm selectively optimizes components of the decision variable, greatly enhancing efficiency and scalability while provably converging to the same solution of our proposed KIO model.

    \item \textbf{Open Source Code}: To foster reproducibility and further research, we provide an open-source implementation of the proposed KIO model and the SSO algorithm, along with the source code of the experiments in Github\footnote{\url{https://github.com/Longyouyuan/Scalable-Kernel-Inverse-Optimization}}. 
\end{enumerate}

\paragraph{Notation} $\mathbb{R}_{+}^n$ denotes the space of $n$-dimensional non-negative vectors.
The identity square matrix with dimension $n$ is denoted by $I_n$.
For a symmetric matrix $Q$, the inequality $Q \succeq 0$ (respectively, $Q \succ 0$) means that $Q$ is positive semi-definite (respectively, positive definite).
The trace of a matrix $Q$ is denoted as $\text{Tr}(Q)$.
Given a vector $x \in \mathbb{R}^n$, we use the shorthand notation $\|x\|_Q^2:=x^{\top} Q x$.
Symmetric block matrices are described by the upper diagonal elements while the lower diagonal elements are replaced by ``*''.
The Frobenius norm of matrix $Q$ is denoted as $\| Q \|_F$.
The notation \(Q_{ij}\) represents the element in the \(i\)-th row and \(j\)-th column of the matrix \(Q\).
The Euclidean inner product of $x$ and $y$ is denoted as $x^\top y$.

\section{Preliminaries}

\subsection{Inverse Optimization}
In general, to solve a data-driven IO problem, we need to design two components: the Forward Optimization Problem (FOP) and the loss function.
Specifically, the FOP corresponds to the optimization problem we aim to ``fit'' to the observed dataset $ \mathcal{\hat{D}} = \{(\hat{s}_i,\hat{u}_i)\}_{i=1}^N$, where each input-output pair $(\hat{s}_i, \hat{u}_i) \in \mathcal{S} \times \mathbb{R}^n$.
In this paper, we use the ``hat'' notation (e.g., $\hat{s}$) to denote objects that depend on the dataset.
Our goal is to find a parameter vector $\theta \in \Theta$ such that
\begin{equation}
    \label{eq:FOP}
    \text{FOP}(\theta, \hat{s}) := \min_{u \in \mathbb{U}(\hat{s})} F_\theta(\hat{s}, u),
\end{equation}
replicates the data as closely as possible by minimizing a loss function, akin to classical empirical risk minimization problems.
In this work, we focus on lifting the learning problem based on quadratic FOPs. These FOPs include linear constraints and continuous decision variables, as defined by
\begin{equation}
    \label{eq:hypothesis}
    F_\theta(\hat{s}, u) := u^\top \theta_{uu} u + 2 \phi(\hat{s})^\top \theta_{su} u \quad \text{and} \quad \mathbb{U}(\hat{s}) := \left\{u \in \mathbb{R}^n : M(\hat{s}) u \leq W(\hat{s}) \right\},
\end{equation}
where $\theta := (\theta_{uu}, \theta_{su})$, $\theta_{uu} \in \mathbb{R}^{n \times n}$, $\theta_{su} \in \mathbb{R}^{m \times n}$, $M(s) \in \mathbb{R}^{m \times n}$, $W(s) \in \mathbb{R}^m$, and $\phi : \mathcal{S} \to \mathbb{R}^m$ is the feature function that maps $\hat{s}$ to a higher-dimensional feature space to enhance the model's capacity.
To simplify notation, we omit the explicit dependence of $M$ and $W$ on $s$, denoting them as $\hat{M}$ and $\hat{W}$, respectively.

To learn $\theta$, we solve a regularized loss minimization problem using the Suboptimality Loss \cite{mohajerin2018data}
\begin{equation}
    \label{eq:loss_minimization}
    \min_{\theta \in \Theta} \ k \mathcal{R}(\theta) + \frac{1}{N} \sum_{i = 1}^N \max_{u_i \in \mathbb{U}(\hat{s}_i)} \big\{ F_\theta(\hat{s}_i, \hat{u}_i) - F_\theta(\hat{s}_i, u_i) \big\},
\end{equation}
where $\Theta := \left\{\theta = (\theta_{uu}, \theta_{su}) : \theta_{uu} \succeq I_n \right\}$, $\mathcal{R}(\theta) := \|\theta_{uu}\|^2_F + \|\theta_{su}\|^2_F$, and $k$ is a positive regularization parameter.
The constraint $\theta_{uu} \succeq I_n$ prevents the trivial solution $\theta_{uu} = \theta_{su} = 0$ and guarantees that the resulting FOP is a convex optimization problem.
Moreover, since $F_\theta$ is linear in $\theta$, the optimization program \eqref{eq:loss_minimization} is convex w.r.t. $\theta$, and it can be reformulated from the ``minimax'' form to a single minimization problem.
This reformulation is based on dualizing the inner maximization problems of \eqref{eq:loss_minimization} and combining the resulting minimization problems.
\begin{proposition}[LMI reformulation \citep{akhtar2021learning}]
    \label{prop:LMI}
    For the hypothesis function and feasible set in \eqref{eq:hypothesis}, the optimization program \eqref{eq:loss_minimization} is equivalent to 
    \begin{equation}
        \label{eq:LMI}
        \begin{aligned} 
            \min_{\theta,\lambda_i,\gamma_i} \quad & k\mathcal{R}(\theta) +\frac{1}{N} \sum_{i=1}^N \left( F_\theta(\hat{s}_i,\hat{u}_i) + \frac{1}{4} \gamma_i + \hat{W}_i^\top \lambda_i \right) \\
            \emph{s.t.} \quad \ \ & \theta = (\theta_{uu}, \theta_{su}), \ \theta_{uu} \succeq I_n, \ \lambda_i \in \mathbb{R}_+^{d}, \ \gamma_i \in \mathbb{R} & \forall i \leq N  \\
            & \begin{bmatrix}
            \theta_{uu} & \hat{M}_i^\top \lambda_i + 2\theta_{su}^\top \phi(\hat{s}_i) \\
            * & \gamma_i
            \end{bmatrix} \succeq 0 & \forall i \leq N.
        \end{aligned}
    \end{equation}
\end{proposition}

\subsection{The Kernel Method}
\label{sec:prelim:kernel-method}
The kernel method is a powerful technique used in machine learning and statistics that exploits the structure of data embedded in a higher-dimensional space.
The kernel method has found numerous applications, including Support Vector Machines \citep{cortes1995support}, Kernel Principal Component Analysis \citep{scholkopf1997kernel}, and Kernel Linear Discriminant Analysis \citep{baudat2000generalized}.
The fundamental idea behind the kernel method is to implicitly map input data into a higher-dimensional space without explicitly computing the transformation, thus enabling algorithms to capture complex patterns and non-linear relationships without heavy computational burden \cite{shafieezadeh2019regularization}.
The kernel method generalizes the hypothesis of the optimization problem to a nonlinear function class based on a \emph{Reproducing Kernel Hilbert Space} (RKHS) $\mathcal{H}$.
{
In this work, we primarily focus on lifting the original parametric optimization problem of the form $\min\limits_{\theta \in \Theta} \, \ell \left(\left\{\left(f_\theta(\hat{s}_i), \hat{s}_i,\hat{u}_i\right)\right\}^N_{i=1}, \|\theta\|_F \right)$, where $f_\theta: \mathcal{S} \rightarrow \mathbb{R}^n$ such that $f_{\theta_{su}}(s)^\top u \coloneqq \phi(s)^\top\theta_{su} u$, following Definition \eqref{eq:hypothesis}.
The lifted problem is then defined as $\min\limits_{f \in \mathcal{H}} \, \ell \left( \left\{\left(f(\hat{s}_i), \hat{s}_i,\hat{u}_i\right)\right\}^N_{i=1}, \|f\|_\mathcal{H} \right)$
with $f$ in an RKHS $\mathcal{H}$.
} 
By lifting the function class to $\mathcal{H}$, we effectively optimize over \textit{nonlinear} hypotheses.

{
In a \emph{vector-valued} RKHS $\mathcal{H}$ equipped with the inner product $\langle \cdot, \cdot \rangle_\mathcal{H}$, as defined in \cite{micchelli2005learning}, given a proper kernel function $K : \mathcal{S} \times \mathcal{S} \rightarrow \mathbb{R}^{n \times n}$ that is symmetric and positive definite, Moore-Aronszajn's reproducing kernels theory implies that there exists a unique RKHS with the reproducing properties induced by \emph{matrix-valued} $K$, such that $\forall f \in \mathcal{H} : f(s)^\top u = \langle f, K(\cdot, s) u \rangle_\mathcal{H}$ with a linear operator $K(\cdot, s): \mathbb{R}^n \rightarrow \mathcal{H}$.
}
Furthermore, the Riesz representation theorem states that for every $s \in \mathcal{S}$ and $u \in \mathbb{R}^n$, there exists a unique function $K(\cdot, s)u \in \mathcal{H}$ for all $f \in \mathcal{H}$.
Recall that the lifted optimization problem is formed over all \textit{nonlinear} hypotheses via $f \in \mathcal{H}$.
However, as a result of the reproducing property \cite{LearningWithKernelsBook}, we can write the lifted optimization problem in the infinite-dimensional inner product space.
Hence, the resulting optimization problem has the form 
\begin{equation}
    \label{eq:kernel-linear-opt-problem}
    \begin{aligned} 
        \min\limits_{f \in \mathcal{H}} \ell \left( 
        \left\{\Big(\big\langle f, K(\cdot, \hat{s}_i)u\big\rangle_\mathcal{H}, \hat{s}_i,\hat{u}_i\Big)\right\}^N_{i=1},
        \|f\|_\mathcal{H} \right).
    \end{aligned}
\end{equation}
In what follows, we show that the solution of the optimization problem \eqref{eq:kernel-linear-opt-problem} exists and is finite when the problem is built over a finite dataset $\hat{\mathcal{D}}$ of size $N$.
We leverage the Representer theorem \cite{shafieezadeh2019regularization}, which for an arbitrary loss function, as in \eqref{eq:kernel-linear-opt-problem}, states that the solution admits the kernel representation of the form $f^\star(s) = \sum_{i=1}^N K(\hat{s}_i, s)\alpha_i$ with $\alpha_i \in \mathbb{R}^n$.
This result effectively suggests that optimizing over an infinite-dimensional RKHS has a sparse solution over the linear hypotheses.

\section{Kernel Inverse Optimization}

In this section, we extend the Inverse Optimization (IO) model proposed by \cite{akhtar2021learning} to incorporate kernel methods.
We consider a hypothesis class of the form in Equation \eqref{eq:FOP}.
However, kernelizing such hypotheses is not straightforward.
Instead, we argue that by kernelizing the optimization problem associated with the loss function described in Equation \eqref{eq:loss_minimization}, we can obtain a forward optimization problem (FOP) that minimizes the kernelized objective $F(s, u)$.

{
In Theorem \ref{theo:kernel_reformulation}, we dualize the problem in \eqref{eq:loss_minimization} and show that the optimal solution for $\theta^\star_{su}$, when plugged into $\mathrm{FOP}(\theta_{su}^\star, \theta_{uu}, \hat{s})$, admits an affine function of $N$ coefficients w.r.t $\hat{s}$.
The resulting $\mathrm{FOP}$, obtained by Theorem \ref{theo:kernel_reformulation}, takes a form consistent with the representer theorem when the loss function and kernel are defined appropriately, as discussed later in this section.
}
\begin{theorem}[Kernel reformulation]
    \label{theo:kernel_reformulation}
    The Lagrangian dual of the optimization program \eqref{eq:LMI} is
    \begin{equation}
        \label{eq:dual}
        \begin{aligned} 
            \min_{P, \Lambda_i, \Gamma_i} \quad & \frac{1}{4k} \left\| \left( \sum_{i=1}^N \frac{\hat{u}_i\hat{u}_i^\top}{N} - \Lambda_i \right) - P \right\|_F^2  - {\rm Tr}(P) \\
            & + \frac{1}{k} \sum_{i=1}^N\sum_{j=1}^N \kappa(\hat{s}_i,\hat{s}_j) \left( \frac{\hat{u}_i}{N} - 2\Gamma_i \right)^\top \left( \frac{\hat{u}_j }{N} - 2\Gamma_j \right) \\
            \emph{s.t.} \quad \ \ & P \succeq 0, \ \Lambda_i \succeq 0, \ \Gamma_i \in \mathbb{R}^n & \forall i \leq N \\
            & \frac{\hat{s}_i}{N} - 2\hat{M}_i\Gamma_i \geq 0 & \forall i \leq N  \\
            & \begin{bmatrix}
            \Lambda_{i} & \Gamma_{i} \\
            * & \frac{1}{4N}
            \end{bmatrix} \succeq 0 & \forall i \leq N,
        \end{aligned}
    \end{equation}
    where $\kappa(\hat{s}_i,\hat{s}_j) = \phi(\hat{s}_i)^\top \phi(\hat{s}_j)$ is the \emph{scalar-valued} kernel function.
    The primal variables $\theta_{uu}$ and $\theta_{su}$ can be recovered using
    \begin{equation}
        \label{eq:dual_to_primal}
        \theta_{uu} = \frac{1}{2k}\left(P - \left(  \sum_{i=1}^N \frac{\hat{u}_i\hat{u}_i^\top}{N} - \Lambda_i \right)\right) \quad \emph{and} \quad \theta_{su} = \sum_{i=1}^N\phi(\hat{s}_i)\frac{1}{k}\left( 2\Gamma_i^\top - \frac{\hat{u}_i^\top }{N} \right).
    \end{equation}
\end{theorem}

\begin{proof}
    See Appendix \ref{proof 1}.
\end{proof}

Notice that the complexity of the optimization program \eqref{eq:dual} does not depend on the dimensionality of the feature vector $\phi(\hat{s}_i)$.
Consequently, this allows us to use kernels generated from infinite-dimensional feature spaces, e.g., the Gaussian (a.k.a. radial basis function) kernel $\kappa(\hat{s}_i, \hat{s}_j) = \exp(-\gamma \|\hat{s}_i - \hat{s}_j \|_2^2)$.
Program \eqref{eq:dual} is a convex optimization problem and can be solved using off-the-shelf solvers, such as MOSEK \cite{aps2019mosek}.
Once solved, we can recover the optimal primal variables $\theta_{uu}^\star$ and $\theta_{su}^\star$ from the optimal dual variables $\Lambda_i^\star,\Gamma_i^\star$ and $P^\star$ using \eqref{eq:dual_to_primal}.
Notice that the dimensionality of $\theta_{su}^\star$ depends on $\phi$, meaning it can be an infinite-dimensional matrix.
However, our ultimate goal is to learn an FOP that replicates the behavior observed in the data. By combining \eqref{eq:dual_to_primal} with \eqref{eq:FOP}, we have that, for the signal $(\hat{s}_\text{new}, \hat{M}_\text{new}, \hat{W}_\text{new})$, the resulting FOP is
\begin{equation}
    \label{eq:FOP_kernel}
    \min_{\hat{M}_\text{new} u \leq \hat{W}_\text{new}} u^\top \frac{1}{2k} \left( P^\star - \left( \sum_{i=1}^N \frac{\hat{u}_i\hat{u}_i^\top}{N} - \Lambda_i^\star \right) \right) u + \sum_{i=1}^N \kappa(\hat{s}_\text{new}, \hat{s}_i)\frac{2}{k} \left( 2\Gamma_i^\star- \frac{\hat{u}_i }{N} \right)^\top u,
\end{equation}
which again does not depend on the dimensionality of $\phi$, but only on the kernel function $\kappa$.
However, a key difference between solving the IO problem using the kernel reformulation of Theorem \ref{theo:kernel_reformulation} and other classical IO approaches (e.g., \citep{mohajerin2018data, akhtar2021learning, zattoniscroccaro2024learning}) is that the resulting FOP \eqref{eq:FOP_kernel} explicitly depends on the entire training dataset $\hat{\mathcal{D}}$.
These models are sometimes called \textit{nonparametric models} \cite{murphy2022probabilistic}, indicating that the number of parameters of the model (in our case, $P^\star$, $\Lambda_i^\star$, and $\Gamma_i^\star$ for all $i \leq N$) depends on the size of the training dataset.

\begin{remark}[A potential variant of representer theorem]
    In the primal problem \eqref{eq:LMI}, a regularized empirical risk loss is optimized over a set of constraints.
    In the learned objective function of the FOP~\eqref{eq:FOP_kernel}, the term related to the features $\hat{s}_i$ (the linear coefficient of the optimizer $u$) can be represented as a finite linear combination of kernel products evaluated on the input points in the training set data, i.e., { \( f^\star(\cdot) = \sum_{i=1}^N \kappa (\hat{s}_i, \cdot) \alpha_i \) with $\alpha \in \mathbb{R}^n$. A similar forward optimization problem can be obtained using a loss function defined in a \emph{vector-valued} RKHS $\mathcal{H}$ with a corresponding kernel function $K$ for functions $f: \mathcal{S} \rightarrow \mathbb{R}^n$, as discussed in Section \ref{sec:prelim:kernel-method}, via
    \[ \ell_{\mathcal{H}}(f, \theta_{uu}, \hat{\mathcal{D}}) = k\|f\|_{\mathcal{H}} + \frac{1}{N}\sum_{i=1}^N \max\limits_{u_i \in \mathbb{U}(\hat{s_i})} \big\{F(\hat{s}_i, \hat{u}_i; \theta_{uu}, f) - F(\hat{s}_i, u_i; \theta_{uu}, f)\big\}, \]
    where 
    $ F(s, u; \theta_{uu}, f) = u^\top\theta_{uu}u + \langle f, K(\cdot, s) u\rangle_{\mathcal{H}} $
    and $K(s, s') \in \mathbb{R}^n \times \mathbb{R}^n$ is set to be a diagonal matrix, with diagonal entries corresponding to the same scalar kernel $\kappa$, such that:
    $ K(s, s')_{jj} = \kappa(s, s')$ for $j \in \{1, \ldots, n\}.$
    Based on the representer theorem, the solution to the optimization problem 
    $ \min\limits_{f \in \mathcal{H}} \, \ell_{\mathcal{H}}(f, \theta_{uu}, \hat{\mathcal{D}}) $
    admits the form
    $ f^\star(s) = \sum_{i=1}^N K(s, \hat{s}_i)\alpha_{i} = \sum_{i=1}^N \kappa(s, \hat{s}_i)\alpha_{i}$.
    This result indicates that the learned FOP \eqref{eq:FOP_kernel} exhibits characteristics consistent with the representer theorem, implying a potential variant in the context of inverse optimization.}
\end{remark}

In the class of cost functions studied in this paper \eqref{eq:hypothesis}, $\theta_{uu}$ can be interpreted as a matrix that penalizes the components of the decision vector $u$.
However, in many problems, it is assumed that the expert generating the data equally penalizes each dimension of $u$, or equivalently, uses $\theta_{uu} = I_n$.
This assumption holds, for instance, in the Gymnasium MuJoCo environments \cite{towers_gymnasium_2023}, where the reward settings for all tasks apply the same penalty to each dimension of the decision vector.
Intuitively, this means that the expert trained under such reward settings aims to reduce the magnitude of the decision vector uniformly across all dimensions, rather than favoring any specific dimension.
Therefore, assuming $\theta_{uu} = I_n$ as prior knowledge can reduce model complexity and lead to faster training, without degrading model performance.

\begin{corollary}
[Kernel reformulation for $\theta_{uu} = I_n$]
\label{corrolary:kernel_reformulation_simpler}
The Lagrangian dual of the optimization program \eqref{eq:LMI} with $\theta_{uu} = I_n$ is 
\begin{equation}
    \label{eq:dual_simpler}
    \begin{aligned} 
        \min_{\Lambda_i, \Gamma_i : \ \forall i \in S} \quad & \frac{1}{k} \sum_{i=1}^N\sum_{j=1}^N \kappa(\hat{s}_i,\hat{s}_j) \left( \frac{\hat{u}_i}{N} - 2\Gamma_i \right)^\top \left( \frac{\hat{u}_j }{N} - 2\Gamma_j \right)  + \sum_{i=1}^N{\rm Tr}(\Lambda_i) \\
        \emph{s.t.} \quad \ \ & \Lambda_i \succeq 0, \ \Gamma_i \in \mathbb{R}^n & \forall i \in S \\
        & \frac{\hat{W}_i}{N} - 2\hat{M}_i\Gamma_i \geq 0 & \forall i \in S  \\
        & \begin{bmatrix}
        \Lambda_{i} & \Gamma_{i} \\
        * & \frac{1}{4N}
        \end{bmatrix} \succeq 0 & \forall i \in S,
    \end{aligned}
\end{equation}
where $\kappa(\hat{s}_i,\hat{s}_j) = \phi(\hat{s}_i)^\top \phi(\hat{s}_j)$ is the kernel function and $S=\{1, \ldots, N\}$.
The primal variable $\theta_{su}$ can be recovered using \eqref{eq:dual_to_primal}.
\end{corollary}

The proof of Corollary \ref{corrolary:kernel_reformulation_simpler} is the same as the proof of Theorem \ref{theo:kernel_reformulation} under the assumption that $\theta_{uu}=I_n$, and is therefore omitted here.
$S$ represents an index set, where all decision variables whose indices belong to $S$ will be optimized, while decision variables whose indices do not belong to $S$ retain their original values and are treated as constants.
In Corollary \ref{corrolary:kernel_reformulation_simpler}, all variables will be optimized, so $S$ is the set comprising natural numbers from $1$ to $N$.
The concept of the index set $S$ is introduced to make Problem \eqref{eq:dual_simpler} compatible with the sub-optimization problems based on the coordinate descent method outlined in Section \ref{section sso}.

In the following section, we will focus on algorithms to solve Problem \eqref{eq:dual_simpler}.
However, all ideas also apply to the general problem \eqref{eq:dual}.

\section{Sequential Selection Optimization}
\label{section sso}
Solving the kernel IO problem \eqref{eq:dual_simpler} involves optimizing a Semidefinite Program (SDP), which can become prohibitively costly if the number of semidefinite constraints and optimization variables grows too large.
In our case, the size of the SDP grows quadratically with the size of the training dataset $N$.
For instance, in our experiments, solving \eqref{eq:dual_simpler} with $N = 20000$ using CVXPY \cite{diamond2016cvxpy} requires up to 256GB of memory.
Therefore, in this section, we propose a coordinate descent-type algorithm to find an approximate solution to \eqref{eq:dual_simpler} by iteratively optimizing only a subset of the coordinates at each iteration, keeping all other coordinates fixed.
We define a pair of variables \(\Lambda_i\) and \(\Gamma_i\) as the $i$-th coordinate, denoted as \(\{\Lambda_i, \Gamma_i\}\).
In Problem \eqref{eq:dual_simpler}, each coordinate is decoupled in the constraints, which enables the use of the coordinate descent framework here.
We call this method \textit{Sequential Selection Optimization} (SSO), and present it in Algorithm \ref{alg:SSO}.

\begin{algorithm}
    \caption{Sequential Selection Optimization (SSO)}
    \begin{algorithmic}[1]
        \State Initialize $\{\Lambda_i, \Gamma_i\}_{i=1}^N$
        \For{$t = 1, \ldots, T$}
            \State Select a batch of $p$ coordinates $S = \{a_i\}_{i=1}^p$, where $a_i \in \{1, \ldots, N\}$
            \State Update $\{ \Lambda_{a_i}, \Gamma_{a_i}\}_{i=1}^p$ based on \eqref{eq:dual_simpler} with $S$
        \EndFor
    \end{algorithmic}
    \label{alg:SSO}
\end{algorithm}

Here, we explain each step of Algorithm \ref{alg:SSO}: (i) Initialization of the optimization variables.
In general, the variables are usually initialized randomly or set to 0 or 1.
These methods are simple but may not provide a good initial guess.
(ii) Selection of a batch of $p$ coordinates.
The most straightforward approach to selecting $p$ coordinates is to choose them cyclically.
Alternatively, we can select the coordinates at random at each iteration (not necessarily with equal probability).
Lastly, we can choose coordinates greedily, selecting the components corresponding to the greatest descent or those with the largest gradient or subgradient at the current iteration \cite{shi2016primer}.
(iii) Solving the KIO subproblem to update the selected coordinates.
The mathematical expression of the subproblem is Problem \eqref{eq:dual_simpler} with $S$, where $S$ is a set containing the indices of the coordinates that need to be updated.
Note that the coordinates whose indices $\notin S$ remain fixed.
Therefore, the number of quadratic terms in \eqref{eq:dual_simpler} scales with $|S|^2$ rather than $N^2$.

Next, we propose two heuristics to accelerate the convergence speed of the SSO algorithm:
a heuristic method for choosing which coordinates (line 3 of Algorithm \ref{alg:SSO}) to optimize, and a warm-up trick to improve the initialization of the optimization variables (line 1 of Algorithm \ref{alg:SSO}).

\subsection{Heuristic for choosing coordinates}
\label{sec:heuristic}

At each iteration of Algorithm \ref{alg:SSO}, intuitively, the largest improvement will be made by updating the ``least optimal'' set of $p$ variables.
One way to evaluate their degree of suboptimality is to choose the variables with the most significant violation of the Karush-Kuhn-Tucker (KKT) conditions of the primal version of Program \eqref{eq:dual_simpler} (i.e., \eqref{eq:LMI} with $\theta_{uu} = I_n$), inspired by the Sequential Minimal Optimization (SMO) method from \cite{platt1998sequential}.

\begin{proposition}[]
    \label{theo:heuristic selection}
    For optimal decision variables of Problem \eqref{eq:dual_simpler}, the coordinate $\{ \Lambda_{i}, \Gamma_{i} \}$ that satisfies 
    \begin{equation}
        \label{eq:condition1}
        \frac{\hat{W}_i}{N} - 2\hat{M}_i\Gamma_i > 0,
    \end{equation}
    should also satisfy
    \begin{equation}
        \label{eq:condition2} 
        {\rm Tr} \left(\begin{bmatrix}
        \Lambda_{i} & \Gamma_{i} \\
        * & \frac{1}{4N}
        \end{bmatrix}
        \begin{bmatrix}
        I_n & 2 \theta_{s u}^{\top} \phi(\hat{s}_i) \\
        * & \|2 \theta_{s u}^{\top}\phi(\hat{s}_i)\|_2^2
        \end{bmatrix} \right)=0.
    \end{equation}
\end{proposition}
\begin{proof}
    See Appendix \ref{proof 2}.
\end{proof}

Proposition \ref{theo:heuristic selection} is based on KKT conditions.
Based on Condition \eqref{eq:condition2}, we can define the KKT violation condition as
\begin{equation}
    \label{eq:KKT_violator} 
    \texttt{kkt\_violator(i)} \coloneqq \left| \text{Tr} \left(\begin{bmatrix}
    \Lambda_{i} & \Gamma_{i} \\
    * & \frac{1}{4N}
    \end{bmatrix}
    \begin{bmatrix}
    I_n & 2 \theta_{s u}^{\top} \phi(\hat{s}_i) \\
    * & \|2 \theta_{s u}^{\top}\phi(\hat{s}_i)\|_2^2
    \end{bmatrix} \right) \right|.
\end{equation}
Using the violation condition \eqref{eq:KKT_violator}, we can establish the following heuristic method to construct the set $S$ in Algorithm \ref{alg:SSO}: given the current values of $\{\Lambda_i,\Gamma_i\}_{i=1}^N$ at iteration $t$, we choose the $p$ coordinates that satisfy Condition \eqref{eq:condition1} with the maximum KKT violation \eqref{eq:KKT_violator}.
In practice, we additionally select some random coordinates to update at each iteration, ensuring that all coordinates have the chance to be updated, including those that initially do not meet the criteria specified in Condition \eqref{eq:condition1} and would otherwise never be updated.
This random selection is inspired by the proven convergence of coordinate descent algorithms with uniformly random and cyclic coordinate selection \cite{nesterov2012efficiency, bertsekas2015convex}.

\subsection{Warm-up trick for improved initialization}
\label{sec:warm-up}

Another component of Algorithm \ref{alg:SSO} that may have a significant practical impact is how the optimization variables $\{\Lambda_i, \Gamma_i\}_{i=1}^N$ are initialized. A poor initial guess (e.g., $\Lambda_i = \Gamma_i = 0$) can lead to slow solver convergence or even result in numerical instability. Here, we propose a simple warm-up trick that leads to a better initialization of the optimization variables. First, we divide the original dataset $\hat{\mathcal{D}}$ into $n$ non-overlapping sub-datasets $\hat{\mathcal{D}}_1, \ldots, \hat{\mathcal{D}}_n$ and solve the $n$ small problems \eqref{eq:dual_simpler} for each of these sub-datasets ($N_i = |\hat{\mathcal{D}}_i|$ and $S$ is the set of indices of all the data in $\hat{\mathcal{D}}_i$). We then concatenate the optimal solutions of these $n$ solved small problems to form an initial guess. Even when the original problem \eqref{eq:dual_simpler} is intractable due to a large training dataset (i.e., large $N$), each subproblem remains tractable for a small enough batch size $N_i$, and its solutions are still feasible with respect to \eqref{eq:dual_simpler}.

\section{Numerical Experiments }
\label{Numerical Experiments}

\subsection{Performance Evaluation}
In this evaluation, KIO is implemented in its simplified version \eqref{eq:dual_simpler}, incorporating a Gaussian kernel, and tested on continuous control datasets from the D4RL benchmark \cite{fu2020d4rl}.
The model is trained using the SSO Algorithm \ref{alg:SSO}.
In each task, the model's performance is assessed over 100 test episodes, and the score\footnote{Regarding the definition of the score for one episode, we refer readers to the official documentation of Gymnasium \cite{towers_gymnasium_2023} and the D4RL paper \cite{fu2020d4rl}.} for KIO is the average score across these 100 episodes.
The parentheses following KIO scores indicate the amount of data used.

For comparison, four additional agents are selected for this experiment.
\textbf{IO} is the inverse optimization model without the kernel method, introduced in Proposition \ref{prop:LMI}.
To illustrate the effect of the kernel method, both the \textbf{KIO} and \textbf{IO} models are trained on identical datasets.
The scores of two behavior cloning agents, \textbf{BC(TD3+BC)} \cite{fujimoto2021minimalist} and \textbf{BC(CQL)} \cite{kumar2020conservative}, are taken from two offline reinforcement learning algorithms in which the entire dataset of 1 million samples was used for training.
These papers implemented their respective behavior cloning agents using D4RL datasets, serving as baselines to compare against their proposed offline reinforcement learning algorithms.
In these studies, \textbf{BC(TD3+BC)} and \textbf{BC(CQL)} were evaluated over $10$ seeds and $3$ seeds, respectively.
The \textbf{Teacher} is the agent responsible for generating the dataset and serves as the target for imitation learning in this experiment.

\begin{table}[!h]
 \caption{Performance of KIO, IO, two Behavior Cloning (BC) agents, and the Teacher agent on MuJoCo tasks from the D4RL benchmark on the normalized return metric.
 The numbers in parentheses represent the amount of data used by KIO and IO, and the score for KIO in each task is the average score over 100 episodes.}
 \label{results1}
 \centering
 \begin{tabular}{cccccc}
    \toprule
    \textbf{Task} & \textbf{KIO} & \textbf{IO} & \textbf{BC(TD3+BC)}\cite{fujimoto2021minimalist} & \textbf{BC(CQL)}\cite{kumar2020conservative} & \textbf{Teacher}\\
    \midrule 
    Hopper-expert & \textbf{109.9}\,(5k) & 31.8 & \textbf{111.5} & \textbf{109.0} & \textbf{108.5}\\
    Hopper-medium & \textbf{50.2}\,(5k) & 20.6 & 30.0 & 29.0 & 47.2\\
    Walker2d-expert & 108.5\,(10k) & 0.9 & 56.0 & \textbf{125.7} & 107.1\\     
    Walker2d-medium & \textbf{74.6}\,(5k) & 0.0 & 11.4 & 6.6 & 68.1\\
    Halfcheetah-expert & 84.4\,(10k) & -1.7 & 105.2 & \textbf{107.0} & 88.1\\
    Halfcheetah-medium & \textbf{39.0}\,(5k) & -3.1 & 36.6 & 36.1 & \textbf{40.7}\\
    \bottomrule 
 \end{tabular}
\end{table}

\paragraph{Evaluation for KIO.} Table \ref{results1} presents the final experimental results, where KIO achieves competitive scores in four out of six tasks.
In these tasks, except for a slightly lower score in the Halfcheetah-expert task compared to the teacher agent, KIO's scores are either close to or exceed those of the teacher agent, indicating strong learning capabilities in complex control tasks.
However, without the kernel method, the IO model demonstrates weak learning capabilities, achieving low scores in the Hopper task and failing to learn in the other two more challenging tasks.
We argue that the weak performance of the IO model is due to the limitations of its hypothesis class, which lacks the richness needed to learn an effective policy for imitation learning tasks.
This limitation arises from its reliance on predefined feature spaces, which may fail to capture the complexities of more sophisticated environments.
All hyperparameters used in this experiment for KIO are listed in Appendix \ref{B-1}.

\makeatletter
\newcommand\figcaption{\def\@captype{figure}\caption}
\newcommand\tabcaption{\def\@captype{table}\caption}
\makeatother

\begin{figure}[htbp]
    \centering
    \begin{minipage}[c]{0.52\textwidth}
        \centering
        \resizebox{\textwidth}{!}{
        \begin{tabular}{ccccccc}
            \toprule
            \multicolumn{2}{c}{\multirow{2}{*}{\textbf{Task}}} &
            \multicolumn{2}{c}{\textbf{SCS}} &&  \multicolumn{2}{c}{\textbf{SSO}}\\
            \cline{3-4} \cline{6-7}
            \multicolumn{2}{c}{} &\textbf{Obj Value}&\textbf{Score} && \textbf{Obj Value}&\textbf{Score}\\
            \midrule
            \multicolumn{2}{c}{Hopper-expert} &185.219&109.9 && 185.220&110.2\\
            \multicolumn{2}{c}{Hopper-medium} &218.761&50.2 && 218.761&51.8\\
            \multicolumn{2}{c}{Walker2d-expert} &140.121&108.5 && 140.121&109.2\\
            \multicolumn{2}{c}{Walker2d-medium} &151.117&74.6 && 151.117&74.9\\
            \multicolumn{2}{c}{Halfcheetah-expert} &165.041&84.4 && 165.041&83.8\\
            \multicolumn{2}{c}{Halfcheetah-medium} &188.184&39.0 && 188.184&39.7\\
            \bottomrule
        \end{tabular}
        }
        \tabcaption{Final Objective Function Value and Score (average return over 100 evaluations) for SCS \cite{ocpb:16} and SSO (20 iterations for all tasks) algorithms.
        The ultimate Objective Function Values of the two algorithms are nearly identical, yet across the majority of tasks, SSO achieves a slightly higher score compared to SCS.\label{table3.1}}
    \end{minipage}
    \hfill
    \begin{minipage}[c]{0.45\textwidth}
        \centering
        \includegraphics[width=\textwidth]{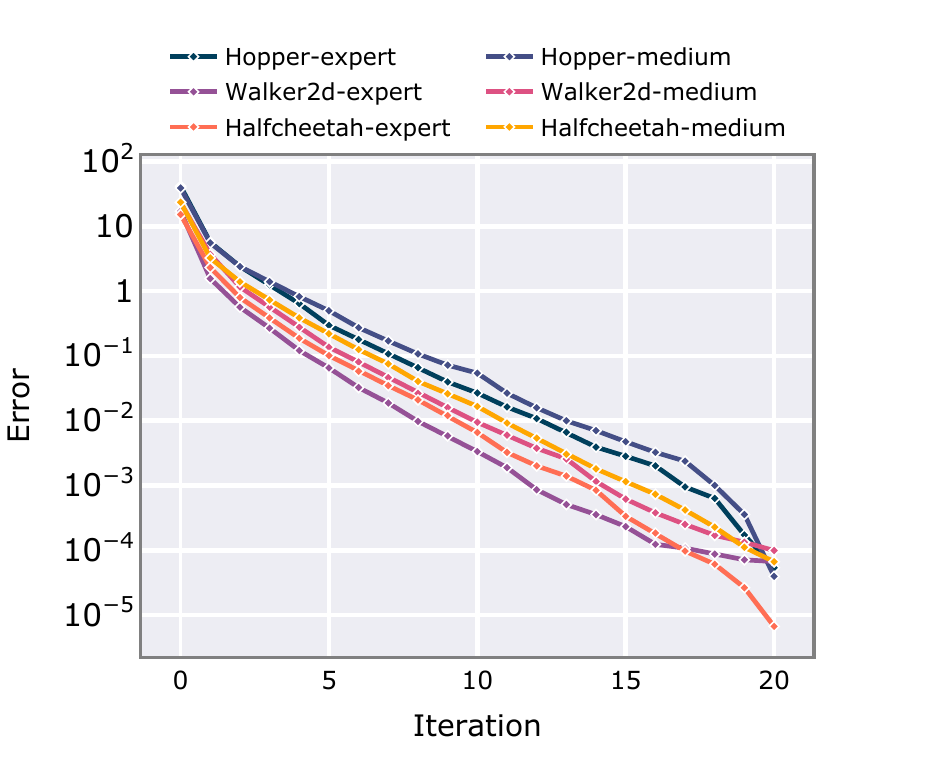} 
        \setlength{\abovecaptionskip}{-0.4cm}
        \figcaption{Convergence curves for SSO.\label{3.1}}
    \end{minipage}
    
\end{figure}

\paragraph{Evaluation for SSO.} Table \ref{table3.1} presents the optimal objective function values and task scores obtained by the centralized algorithm, Splitting Conic Solver (SCS) \cite{ocpb:16}, and the distributed algorithm, Sequential Selection Optimization (SSO), for solving Problem \eqref{eq:dual_simpler}.
In this experiment, SCS is employed to directly address the large-scale problem \eqref{eq:dual_simpler} with $|S|=N$.
The corresponding solution is evaluated 100 times, and the average score is taken as the final result.
Meanwhile, SSO addresses the large-scale problem by solving a series of subproblems.
After each iteration, the current solution is evaluated 100 times, and the average score is recorded.
After 20 iterations, there are 20 corresponding scores, and we select the highest score along with the objective function value from the last iteration and report it.
The results show that SSO and SCS yield nearly identical optimal objective function values.
However, except for the Halfcheetah-medium task, SSO achieved higher scores across all other tasks.
Figure \ref{3.1} displays the convergence performance of the SSO algorithm across six distinct tasks, with the horizontal axis representing the number of iterations and the vertical axis representing the error between the current objective function value and the optimal objective function value (calculated by SCS) in Problem \eqref{eq:dual_simpler}.
The SSO algorithm demonstrates a fast convergence rate.
By the 10th iteration, the errors for all tasks are below 0.1, and by the 20th iteration, the errors have further diminished to approximately 1e-4 for all tasks.

In the previous evaluation of the SSO algorithm, we limited the maximum training data to 10k to ensure that we could directly solve Problem \eqref{eq:dual_simpler} without SSO.
Thereby, we were able to compare the results with and without the SSO algorithm.
However, to further verify the effectiveness of the SSO algorithm, we tested it on a new task (the medium-expert dataset) using 100k data points.
At this scale, due to memory limitations, we were unable to solve Problem \eqref{eq:dual_simpler} directly without the SSO algorithm, thus the effectiveness of the SSO algorithm is inferred solely from its experimental performance.
Table \ref{results3} presents the results of the KIO model optimized by the SSO algorithm.
The results indicate that the KIO model achieves competitive results and scales effectively to larger data sizes.
We list the hyperparameters used in this experiment in Appendix \ref{B-1}.

\begin{table}[!ht]
\caption{Performance of KIO, two Behavior Cloning (BC) agents, and the Teacher agent on MuJoCo tasks from the D4RL benchmark on the normalized return metric.
The numbers in parentheses represent the amount of data used by KIO, and the score for KIO in each task is the average score over 100 episodes.}
\label{results3}
\centering
\begin{tabular}{ccccccc}
    \toprule
    \textbf{Task} & \textbf{KIO} & \textbf{BC(TD3+BC)} & \textbf{BC(CQL)} & \textbf{Teacher}\\
    \midrule 
    Hopper-medium-expert & 79.6\,(100k) & 89.6 & \textbf{111.9} & 64.8\\
    Walker2d-medium-expert & \textbf{100.1}\,(100k)  & 12.0 & 11.3 & 82.7\\
    Halfcheetah-medium-expert & 46.4\,(100k)  & \textbf{67.6} & 35.8 & 64.4\\     
    \bottomrule 
\end{tabular}
\end{table}

\subsection{Ablation Studies}
\label{sec:numerical-ablation}
We perform ablation studies to understand the contribution of each individual component: Heuristic Coordinates Selection (Section \ref{sec:heuristic}) and Warm-Up Trick (Section \ref{sec:warm-up}).
Our results, presented in Figure \ref{ablation}, compare the performance of SSO with and without each component (all model hyperparameters remain unchanged as shown in Appendix \ref{B-1}).

We use the Hopper task as the testing task, with the first 5000 data points from the D4RL Hopper-expert dataset as training data.
The block coordinate for each iteration consists of 2500 coordinates ($|S|=2500$).
When applying the Warm-Up Trick, we partition the data into two equal parts and solve two subproblems \eqref{eq:dual_simpler}, each with $|S|=2500$.
Therefore, the computational time required for the Warm-Up Trick is approximately equal to the time 
\begin{wrapfigure}{r}{0.5\textwidth}
    \centering
    \vspace{-0.3cm}
    \includegraphics[width=0.45\textwidth]{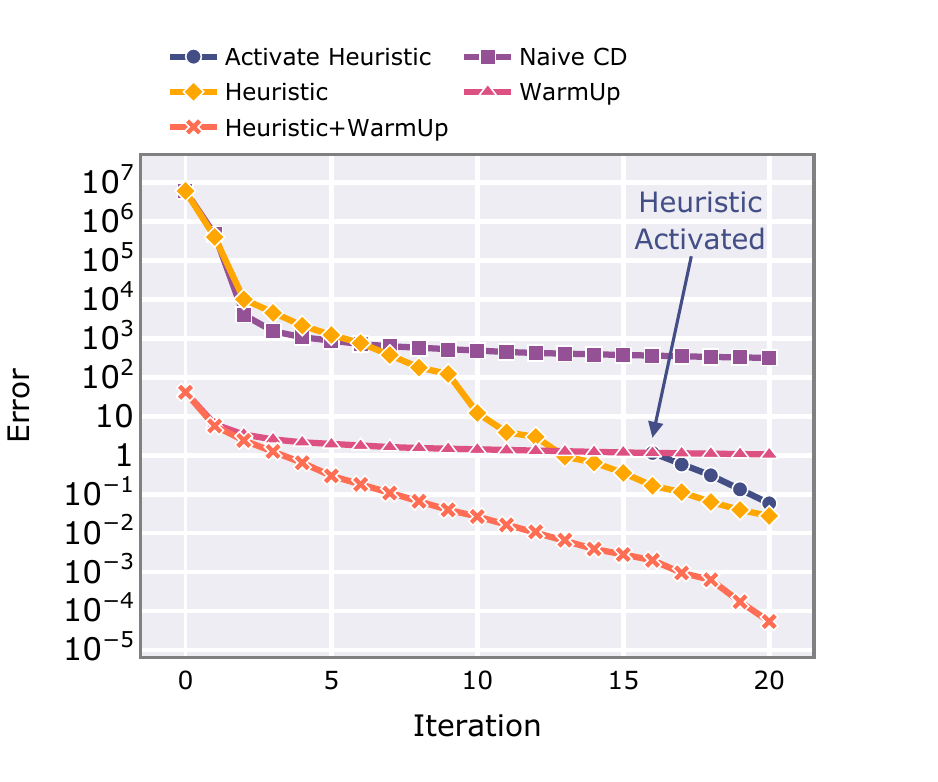} 
    \setlength{\abovecaptionskip}{-0.1cm}
    \caption{Convergence curves on the MuJoCo Hopper task with the first 5k data points from the D4RL Hopper-expert dataset.
    The vertical axis represents the difference between the current objective function value and the optimal value.
    Sequential Selection Optimization (orange) exhibits the fastest convergence rate.
    \label{ablation}
    }
    \label{fig:example}
    \vspace{-0.6cm}
\end{wrapfigure}
needed for two iterations of the SSO algorithm. In Figure \ref{ablation}, we present the results of 20 iterations, with the vertical axis representing the difference between the current objective function value and the optimal value in Problem \eqref{eq:dual_simpler}.

Both the Heuristic Coordinates Selection, abbreviated as Heuristic in Figure \ref{ablation}, and Warm-Up Trick significantly accelerate the algorithm.
With the Warm-Up Trick, the initial objective function value is markedly reduced.
Without the Warm-Up Trick, the Heuristic curve requires approximately 10 iterations to reach the initial values of the SSO curve, whereas the Warm-Up Trick requires only about the time of two iterations.
The Heuristic Coordinates Selection results in rapid descent of the error curves.
The WarmUp curve, however, becomes nearly flat after a few iterations, until the 17th iteration when the Heuristic Coordinates Selection method is activated, causing a rapid decrease in the curve.

\section{Conclusion and Limitations} \label{conclusion}
We introduced Kernel Inverse Optimization (KIO), an inverse optimization model leveraging kernel methods, along with its simplified variant and the theoretical derivations.
Subsequently, we proposed the Sequential Selection Optimization (SSO) algorithm for training the KIO model, which addresses memory issues by decomposing the original problem into a series of subproblems.
Our empirical results demonstrate that KIO exhibits strong learning capabilities in complex control tasks, while the SSO algorithm achieves rapid convergence to the optimal solution within a limited number of iterations.

One of the limitations of this model is the computational cost of adding a new data point.
In that case, all training data are required to compute the coefficients for the FOP problem (see FOP \eqref{eq:FOP_kernel}) for that point.
Thus, as the amount of training data grows, so does the computational cost.
Another limitation is the absence of theoretical analysis on the convergence rate of the SSO algorithm, which we leave for future research.
Furthermore, in our numerical experiments in Section \ref{Numerical Experiments}, the proposed KIO model, even with the SSO algorithm, required substantial memory resources—up to 256 GB when using 100k data points.
However, we believe that further optimization in implementation could reduce these memory requirements.
Finally, in our numerical experiments in Section \ref{Numerical Experiments}, we observed that initialization strategies critically impact the performance of the SSO algorithm.
Thus, exploring alternative initialization strategies beyond the one proposed in Section \ref{sec:warm-up} presents a promising direction for future work.

\appendix


\section*{SUPPLEMENTARY}

\section{Technical Proofs}
\label{app:proofs}

\subsection{Proof of Theorem \ref{theo:kernel_reformulation} \label{proof 1}}
First, let $P$, $\Tilde{\lambda}_i$ and $\begin{bmatrix}
\Lambda_{i}&\Gamma_{i}\\
*&\alpha_{i}
\end{bmatrix}$ be the Lagrange multiplier associated with the constraints $\theta_{u u} \succeq I_m$, $ \lambda_i \in \mathbb{R}_+^{d}$ and $\left[\begin{array}{cc}
\theta_{u u} & \hat{M}_i^{\top} \lambda_i+2 \theta_{s u}^{\top} \phi(\hat{s}_i) \\
* & \gamma_i
\end{array}\right] \succeq 0$ respectively, where $P,\Lambda_i \in \mathbb{R}^{n \times n}$, $\Gamma_i \in \mathbb{R}^{n}$, $\Tilde{\lambda}_i \in \mathbb{R}^{d}$ and $\gamma_i \in \mathbb{R}$. Then define the Lagrangian function
\begin{equation}
 \label{eq:lagrangian}
\begin{aligned}   
L(\theta_{uu},\theta_{su},\lambda_i,\gamma_i, P, \Tilde{\lambda}_i,\Lambda_{i},\Gamma_{i},\alpha_{i}) = & k\|\theta_{uu}\|_F^2 + k\|\theta_{su}\|_F^2 -\text{Tr}(P(\theta_{uu}-I_n)) - \sum_{i=1}^N\Tilde{\lambda_i}^\top \lambda_i \\
& + \frac{1}{N}\sum_{i=1}^N \left(\hat{u}_i^\top \theta_{uu}\hat{u}_i + 2\phi(\hat{s}_i)^\top \theta_{su}\hat{u}_i + \frac{1}{4} \gamma_i + \hat{W}_i^{\top} \lambda_i \right) \\
&  - \sum_{i=1}^N \text{Tr}
\left(\begin{bmatrix}
\Lambda_{i} & \Gamma_{i} \\
* & \alpha_{i}
\end{bmatrix}
\begin{bmatrix}
\theta_{uu} & \hat{M}_i^{\top} \lambda_i+2 \theta_{s u}^{\top} \phi(\hat{s}_i) \\
* & \gamma_i
\end{bmatrix} \right),
\end{aligned}
\end{equation}
and the Lagrange dual problem
\begin{equation}
 \label{eq:lagrandian_dual}
\begin{aligned} 
\begin{array}{cl}
\max\limits_{P,\Tilde{\lambda}_i,\Lambda_i,\Gamma_i,\alpha_i} & \inf\limits_{\theta_{uu},\theta_{su},\lambda_i,\gamma_i} L(\theta_{uu},\theta_{su},\lambda_i,\gamma_i, P, \Tilde{\lambda}_i,\Lambda_{i},\Gamma_{i},\alpha_{i}) \\
\text { s.t. } & P \succeq 0,\  \Tilde{\lambda}_i \in R_+^{d}, \quad \forall i \leq N  \\
& \begin{bmatrix}
    \Lambda_{i}&\Gamma_{i}\\
    *&\alpha_{i}
    \end{bmatrix} \succeq 0, \quad \forall i \leq N.
\end{array}
\end{aligned}
\end{equation}
When the Lagrangian function \eqref{eq:lagrangian} is at the point of the infimum with respect to $\theta_{uu},\theta_{su},\lambda_i,\gamma_i$, we have
\begin{equation*}   
\frac{\partial L}{\partial \theta_{uu}}=2k\theta_{uu}+\frac{1}{N}\sum_{i=1}^N\hat{u}_i\hat{u}_i^\top -P+\sum_{i=1}^N-\Lambda_i=0 \Rightarrow \theta_{uu} = \frac{1}{2k}\left(P - \left(\sum_{i=1}^N \frac{\hat{u}_i\hat{u}_i^\top }{N} - \Lambda_i \right)\right)
\end{equation*}
\begin{equation*}  
\frac{\partial L}{\partial \theta_{su}} = 2k\theta_{su}+2\sum_{i=1}^N\phi(\hat{s}_i) \left(\frac{\hat{u}_i^\top }{N}-2\Gamma_i^\top \right) = 0 \Rightarrow \theta_{su} = \sum_{i=1}^N\phi(\hat{s}_i)\frac{1}{k}\left( 2\Gamma_i^\top - \frac{\hat{u}_i^\top }{N} \right)
\end{equation*}
\begin{equation*}   
\frac{\partial L}{\partial \lambda_i}=\frac{\hat{W}_i}{N}-\Tilde{\lambda}_i-2\hat{M}_i\Gamma_i=0 \Rightarrow \Tilde{\lambda}_i=\frac{\hat{W}_i}{N}-2\hat{M}_i\Gamma_i
\end{equation*}
\begin{equation*} 
\frac{\partial L}{\partial \gamma_i}=\frac{1}{4N}-\alpha_i=0 \Rightarrow \alpha_i=\frac{1}{4N}
\end{equation*}
Finally, substituting the expressions for $\theta_{uu}$, $\theta_{su}$, $\Tilde{\lambda}_i$, and $\alpha_i$ into the Lagrange dual problem \eqref{eq:lagrandian_dual} and simplifying completes the proof.

\subsection{Proof of Proposition \ref{theo:heuristic selection}}
\label{proof 2}
Note that Problem \eqref{eq:dual_simpler} is the dual problem of Problem \eqref{eq:LMI} with $\theta_{uu}=I_n$. Here, we enumerate the five KKT conditions that will be employed
\begin{subequations}
\label{eq:KKT}
\begin{flalign}
    \theta_{su} = \frac{1}{k}\sum_{i=1}^N\phi(\hat{s}_i)\left( 2\Gamma_i^\top - \frac{\hat{u}_i^\top}{N} \right),
     && \text{(stationarity)} \label{eq:KKTa} \\
    \Tilde{\lambda}_i=\frac{\hat{W}_i}{N} - 2\hat{M}_i\Gamma_i,\ \forall i \leq N,
     && \text{(stationarity)} \label{eq:KKTb} \\
    \Tilde{\lambda_i}^\top \lambda_i=0,\ \forall i \leq N, && \text{(complementary slackness)} \label{eq:KKTc} \\
    \text{Tr}\left(
    \begin{bmatrix}
    \Lambda_{i} & \Gamma_{i} \\
    * & \frac{1}{4N}
    \end{bmatrix}
    \begin{bmatrix}
    I_n & \hat{M}_i^{\top} \lambda_i+2 \theta_{s u}^{\top} \phi(\hat{s}_i) \\
    * & \gamma_i
    \end{bmatrix} \right) = 0,\ \forall i \leq N,
    && \text{(complementary slackness)} \label{eq:KKTd} \\
    \lambda_i \in R_+^{d},\ \forall i \leq N. && \text{(primal feasibility)} \label{eq:KKTe}
\end{flalign}

\end{subequations}
First, we choose coordinate $i$ that satisfy Condition \eqref{eq:condition1}. Then based on KKT condition \eqref{eq:KKTb}, we have 
\begin{equation}
\begin{aligned}   
\Tilde{\lambda}_i>0.
\end{aligned}
\end{equation}
Next, based on conditions in \eqref{eq:KKTc} and \eqref{eq:KKTe}, one can obtain
\begin{equation}
\begin{aligned}   
\lambda_i=0.
\end{aligned}\label{eq:lamda=0}
\end{equation}
Substituting the result \eqref{eq:lamda=0} into Condition \eqref{eq:KKTd} yields
\begin{equation}
\begin{aligned}   
\text{Tr} \left(\begin{bmatrix}
    \Lambda_{i}&\Gamma_{i}\\
    *&\frac{1}{4N}
    \end{bmatrix}\begin{bmatrix}
    I_n & 2 \theta_{s u}^{\top} \phi(\hat{s}_i) \\
* & \gamma_i
    \end{bmatrix} \right)=0.
\end{aligned}\label{first TR}
\end{equation}
$\gamma_i$ is the decision variable of Problem \eqref{eq:LMI} with $\theta_{uu}=I_n$. Next, let's solve for its expression. By utilizing the Schur complement, we can prove the following two constraints are equivalent
\begin{equation*}
\begin{bmatrix}
I_n & \hat{M}_i^{\top} \lambda_i+2 \theta_{s u}^{\top} \phi(\hat{s}_i) \\
* & \gamma_i
\end{bmatrix} \succeq 0 \Leftrightarrow \gamma_i \geq \left \|\hat{M}_i^{\top} \lambda_i+2 \theta_{s u}^{\top} \phi(\hat{s}_i) \right\|_2^2.
\end{equation*}
Therefore, Problem \eqref{eq:LMI} with $\theta_{uu}=I_n$ can be equivalently expressed as
\begin{equation}
\begin{aligned} 
\begin{array}{cl}
\min\limits_{\theta_{su},\textcolor{red}{\gamma_i},\lambda_i} &k\|\theta_{su}\|_F^2+ \frac{1}{N}\sum_{i=1}^N\left(2\hat{u}_i^T\theta^T_{su}\phi(\hat{s}_i)+\frac{1}{4} \textcolor{red}{\gamma_i}+\hat{W}_i^\top \lambda_i\right) \\
\text { s.t. } & \lambda_i \in \mathbb{R}_+^{d},\ \textcolor{red}{\gamma_i} \in \mathbb{R}, \quad \forall i \leq N  \\
& \textcolor{red}{\gamma_i} \geq \left \|\hat{M}_i^{\top} \lambda_i+2 \theta_{s u}^{\top} \phi(\hat{s}_i) \right\|_2^2, \quad \forall i \leq N,
\end{array}
\end{aligned}\label{non LMI}
\end{equation}
Here, the variable $\gamma_i$ is highlighted. It can be easily proven that when $\gamma_i$ attain its optimal values, the equality in the last constraint should hold: $\gamma_i = \left \|\hat{M}_i^{\top} \lambda_i+2 \theta_{s u}^{\top} \phi(\hat{s}_i) \right\|_2^2$. Note that $\lambda_i=0$ \eqref{eq:lamda=0}, then the expression of $\gamma_i$ is
\begin{equation}
\begin{aligned}   
\gamma_i = \left \|2 \theta_{s u}^{\top} \phi(\hat{s}_i) \right\|_2^2.
\end{aligned}\label{gamma value}
\end{equation}
Substituting \eqref{gamma value} into \eqref{first TR}, one can obtain Condition \eqref{eq:condition2}.

\section{Hyperparameters of KIO\label{B-1}}

Table \ref{table:appx-hyper} provides the hyperparameters used in the experiments in Section \ref{Numerical Experiments}.

\begin{table}[!h]
\caption{KIO Environment Specific Parameters. $N$ is the size of the dataset, \texttt{p} is the number of coordinates that are updated in one iteration of SSO, $k$ is the regularization coefficient, and \texttt{scalar} is the multiplier used in the implementation to avoid numerical instabilities. }
     \centering
     \begin{tabular}{ccccc}
     \midrule 
     Environment & Dataset & $k$ & \texttt{scalar} & \texttt{p}\\
     \midrule 
     Hopper-expert & 0:5k & 1e-6 & 200$N$ & 0.5$N$\\
     Hopper-medium & 0:5k & 1e-6 & 200$N$ & 0.5$N$\\
     Hopper-medium-expert & First 50k + Last 50k & 1e-6 & $200N$ & 0.1$N$\\
     walker2d-expert & 0:5k+10k:15k & 1e-5 & 50$N$ & 0.5$N$\\     
     walker2d-medium & 15k:20k & 1e-5 & 50$N$ & 0.5$N$\\
     walker2d-medium-expert & First 50k + Last 50k & 1e-5 & $50N$ & 0.1$N$\\
     Halfcheetah-expert & 325k:330k+345k:350k & 5e-6 & 50$N$ & 0.5$N$\\
     Halfcheetah-medium & 10k:15k & 5e-6 & 50$N$ & 0.5$N$\\
     Halfcheetah-medium-expert & First 50k + Last 50k & 1e-6 & $50N$ & 0.1$N$\\
     \bottomrule 
     \end{tabular}
\label{table:appx-hyper}
\end{table}

\section{Ablation Study on Kernel Function\label{C-1}}

\begin{table}[H]
 \caption{Performance of KIO on MuJoCo tasks from the D4RL benchmark on the normalized return metric. The scores in each task represent the average score over 100 episodes within the range of one standard deviation.}
\label{table:appx-ablation}
 \centering
 \begin{tabular}{cccccc}
    \toprule
    \textbf{Task}      & \textbf{RBF}           & \textbf{Laplace}   & \textbf{Linear} \\
    \midrule 
    Hopper-medium      & 51 $\pm$ 6.4       & 41 $\pm$ 5.1         & 24 $\pm$ 0.05    \\
    Hopper-expert      & 109.9 $\pm$ 0.4    & 71 $\pm$ 26.7        & 28 $\pm$ 0.5     \\
    Walker2d-medium    & 72 $\pm$ 14        & 43 $\pm$ 28.2        & -0.19 $\pm$ 0.006    \\     
    Walker2d-expert    & 109.1 $\pm$ 0.3    & 103.1 $\pm$ 22.6     & -0.02 $\pm$ 0.1   \\
    Halfcheetah-medium & 32.4 $\pm$ 12.2    & 52 $\pm$ 10.5	       &-0.8 $\pm$ 0.6     \\
    Halfcheetah-expert & 78.8 $\pm$ 24.9    & 59.1 $\pm$ 35.5      & 2.0 $\pm$ 3.1     \\
    \bottomrule 
 \end{tabular}
\end{table}

In addition to the studies presented in Section \ref{sec:numerical-ablation}, we conducted experiments to evaluate our proposed model using different kernels commonly employed in machine learning: Gaussian (RBF), Laplacian, and linear kernels.
The results for the RBF kernel in Table \ref{table:appx-ablation} differ slightly from those in Table \ref{results3} due to using a different dataset partition for normalizing the states.

\ack{This work is partially supported by the European Research Council (ERC) under the European Unions Horizon 2020 research and innovation programme (TRUST-949796).
}

\bibliographystyle{plain}
\bibliography{MyBib}

\begin{thebibliography}{10}

\bibitem{ahuja2001inverse}
Ravindra~K Ahuja and James~B Orlin.
\newblock Inverse optimization.
\newblock {\em Operations research}, 49(5):771--783, 2001.

\bibitem{akhtar2021learning}
Syed~Adnan Akhtar, Arman~Sharifi Kolarijani, and Peyman~Mohajerin Esfahani.
\newblock Learning for control: An inverse optimization approach.
\newblock {\em IEEE Control Systems Letters}, 6:187--192, 2021.

\bibitem{aps2019mosek}
Mosek ApS.
\newblock Mosek optimization toolbox for matlab.
\newblock {\em User’s Guide and Reference Manual, Version}, 4:1, 2019.

\bibitem{aswani2018inverse}
Anil Aswani, Zuo-Jun Shen, and Auyon Siddiq.
\newblock Inverse optimization with noisy data.
\newblock {\em Operations Research}, 66(3):870--892, 2018.

\bibitem{ayer2015inverse}
Turgay Ayer.
\newblock Inverse optimization for assessing emerging technologies in breast cancer screening.
\newblock {\em Annals of operations research}, 230:57--85, 2015.

\bibitem{baudat2000generalized}
Gaston Baudat and Fatiha Anouar.
\newblock Generalized discriminant analysis using a kernel approach.
\newblock {\em Neural computation}, 12(10):2385--2404, 2000.

\bibitem{bertsekas2015convex}
Dimitri Bertsekas.
\newblock {\em Convex optimization algorithms}.
\newblock Athena Scientific, 2015.

\bibitem{bertsimas2015data}
Dimitris Bertsimas, Vishal Gupta, and Ioannis~Ch Paschalidis.
\newblock Data-driven estimation in equilibrium using inverse optimization.
\newblock {\em Mathematical Programming}, 153:595--633, 2015.

\bibitem{birge2017inverse}
John~R Birge, Ali Horta{\c{c}}su, and J~Michael Pavlin.
\newblock Inverse optimization for the recovery of market structure from market outcomes: An application to the miso electricity market.
\newblock {\em Operations Research}, 65(4):837--855, 2017.

\bibitem{chan2023inverse}
Timothy~CY Chan, Rafid Mahmood, and Ian~Yihang Zhu.
\newblock Inverse optimization: Theory and applications.
\newblock {\em Operations Research}, 2023.

\bibitem{chen2021inverse}
Lu~Chen, Yuyi Chen, and Andr{\'e} Langevin.
\newblock An inverse optimization approach for a capacitated vehicle routing problem.
\newblock {\em European Journal of Operational Research}, 295(3):1087--1098, 2021.

\bibitem{cortes1995support}
Corinna Cortes and Vladimir Vapnik.
\newblock Support-vector networks.
\newblock {\em Machine learning}, 20:273--297, 1995.

\bibitem{diamond2016cvxpy}
Steven Diamond and Stephen Boyd.
\newblock {CVXPY}: {A} {P}ython-embedded modeling language for convex optimization.
\newblock {\em Journal of Machine Learning Research}, 17(83):1--5, 2016.

\bibitem{farago2003inverse}
Andr{\'a}s Farag{\'o}, {\'A}ron Szentesi, and Bal{\'a}zs Szviatovszki.
\newblock Inverse optimization in high-speed networks.
\newblock {\em Discrete Applied Mathematics}, 129(1):83--98, 2003.

\bibitem{fernandez2021forecasting}
Ricardo Fern{\'a}ndez-Blanco, Juan~Miguel Morales, and Salvador Pineda.
\newblock Forecasting the price-response of a pool of buildings via homothetic inverse optimization.
\newblock {\em Applied Energy}, 290:116791, 2021.

\bibitem{fernandez2021inverse}
Ricardo Fern{\'a}ndez-Blanco, Juan~Miguel Morales, Salvador Pineda, and {\'A}lvaro Porras.
\newblock Inverse optimization with kernel regression: Application to the power forecasting and bidding of a fleet of electric vehicles.
\newblock {\em Computers \& Operations Research}, 134:105405, 2021.

\bibitem{fu2020d4rl}
Justin Fu, Aviral Kumar, Ofir Nachum, George Tucker, and Sergey Levine.
\newblock D4rl: Datasets for deep data-driven reinforcement learning, 2020.

\bibitem{fujimoto2021minimalist}
Scott Fujimoto and Shixiang~Shane Gu.
\newblock A minimalist approach to offline reinforcement learning.
\newblock {\em Advances in neural information processing systems}, 34:20132--20145, 2021.

\bibitem{iyengar2005inverse}
Garud Iyengar and Wanmo Kang.
\newblock Inverse conic programming with applications.
\newblock {\em Operations Research Letters}, 33(3):319--330, 2005.

\bibitem{keshavarz2011imputing}
Arezou Keshavarz, Yang Wang, and Stephen Boyd.
\newblock Imputing a convex objective function.
\newblock In {\em 2011 IEEE international symposium on intelligent control}, pages 613--619. IEEE, 2011.

\bibitem{kumar2020conservative}
Aviral Kumar, Aurick Zhou, George Tucker, and Sergey Levine.
\newblock Conservative q-learning for offline reinforcement learning.
\newblock {\em Advances in Neural Information Processing Systems}, 33:1179--1191, 2020.

\bibitem{li2021inverse}
Jonathan Yu-Meng Li.
\newblock Inverse optimization of convex risk functions.
\newblock {\em Management Science}, 67(11):7113--7141, 2021.

\bibitem{micchelli2005learning}
Charles~A Micchelli and Massimiliano Pontil.
\newblock On learning vector-valued functions.
\newblock {\em Neural computation}, 17(1):177--204, 2005.

\bibitem{mohajerin2018data}
Peyman Mohajerin~Esfahani and Daniel Kuhn.
\newblock Data-driven distributionally robust optimization using the wasserstein metric: performance guarantees and tractable reformulations.
\newblock {\em Mathematical Programming}, 171(1-2):115--166, 2018.

\bibitem{murphy2022probabilistic}
Kevin~P. Murphy.
\newblock {\em Probabilistic machine learning: an introduction}.
\newblock MIT press, 2022.

\bibitem{nesterov2012efficiency}
Yu~Nesterov.
\newblock Efficiency of coordinate descent methods on huge-scale optimization problems.
\newblock {\em SIAM Journal on Optimization}, 22(2):341--362, 2012.

\bibitem{ocpb:16}
Brendan O'Donoghue, Eric Chu, Neal Parikh, and Stephen Boyd.
\newblock Conic optimization via operator splitting and homogeneous self-dual embedding.
\newblock {\em Journal of Optimization Theory and Applications}, 169(3):1042--1068, June 2016.

\bibitem{patriksson2015traffic}
Michael Patriksson.
\newblock {\em The traffic assignment problem: models and methods}.
\newblock Courier Dover Publications, 2015.

\bibitem{platt1998sequential}
John Platt.
\newblock Sequential minimal optimization: A fast algorithm for training support vector machines.
\newblock Technical report, Microsoft, 1998.

\bibitem{ronnqvist2017calibrated}
Mikael R{\"o}nnqvist, Gunnar Svenson, Patrik Flisberg, and Lars-Erik J{\"o}nsson.
\newblock Calibrated route finder: Improving the safety, environmental consciousness, and cost effectiveness of truck routing in sweden.
\newblock {\em Interfaces}, 47(5):372--395, 2017.

\bibitem{saez2016data}
Javier Saez-Gallego, Juan~M Morales, Marco Zugno, and Henrik Madsen.
\newblock A data-driven bidding model for a cluster of price-responsive consumers of electricity.
\newblock {\em IEEE Transactions on Power Systems}, 31(6):5001--5011, 2016.

\bibitem{scholkopf1997kernel}
Bernhard Sch{\"o}lkopf, Alexander Smola, and Klaus-Robert M{\"u}ller.
\newblock Kernel principal component analysis.
\newblock In {\em International conference on artificial neural networks}, pages 583--588. Springer, 1997.

\bibitem{LearningWithKernelsBook}
Bernhard Scholkopf and Alexander~J. Smola.
\newblock {\em Learning with Kernels: Support Vector Machines, Regularization, Optimization, and Beyond}.
\newblock MIT Press, Cambridge, MA, USA, 2001.

\bibitem{shafieezadeh2019regularization}
Soroosh Shafieezadeh-Abadeh, Daniel Kuhn, and Peyman~Mohajerin Esfahani.
\newblock Regularization via mass transportation.
\newblock {\em Journal of Machine Learning Research}, 20(103):1--68, 2019.

\bibitem{shahmoradi2022quantile}
Zahed Shahmoradi and Taewoo Lee.
\newblock Quantile inverse optimization: Improving stability in inverse linear programming.
\newblock {\em Operations research}, 70(4):2538--2562, 2022.

\bibitem{shi2016primer}
Hao-Jun~Michael Shi, Shenyinying Tu, Yangyang Xu, and Wotao Yin.
\newblock A primer on coordinate descent algorithms.
\newblock {\em arXiv preprint arXiv:1610.00040}, 2016.

\bibitem{towers_gymnasium_2023}
Mark Towers, Jordan~K. Terry, Ariel Kwiatkowski, John~U. Balis, Gianluca~de Cola, Tristan Deleu, Manuel Goulão, Andreas Kallinteris, Arjun KG, Markus Krimmel, Rodrigo Perez-Vicente, Andrea Pierré, Sander Schulhoff, Jun~Jet Tai, Andrew Tan~Jin Shen, and Omar~G. Younis.
\newblock Gymnasium, March 2023.

\bibitem{yu2023learning}
Shi Yu, Haoran Wang, and Chaosheng Dong.
\newblock Learning risk preferences from investment portfolios using inverse optimization.
\newblock {\em Research in International Business and Finance}, 64:101879, 2023.

\bibitem{zattoniscroccaro2024learning}
Pedro Zattoni~Scroccaro, Bilge Atasoy, and Peyman Mohajerin~Esfahani.
\newblock Learning in inverse optimization: Incenter cost, augmented suboptimality loss, and algorithms.
\newblock {\em Operations Research}, 2024.

\bibitem{zattoniscroccaro2024inverse}
Pedro Zattoni~Scroccaro, Piet van Beek, Peyman Mohajerin~Esfahani, and Bilge Atasoy.
\newblock Inverse optimization for routing problems.
\newblock {\em Transportation Science}, 2024.

\end{thebibliography}

\end{document}